\documentclass{article}%
\usepackage{amsmath}
\usepackage{amsfonts}
\usepackage{amssymb,float}
\usepackage{latexsym}
\usepackage[latin1]{inputenc}
\usepackage{amsmath}
\usepackage{amssymb,float}
\usepackage{latexsym}
\usepackage{epstopdf}
\usepackage{geometry}
\usepackage[active]{srcltx}
\usepackage{graphicx}
\usepackage{epstopdf}
\usepackage{enumerate}
\usepackage[ bookmarks=true,         bookmarksnumbered=true, colorlinks=true,
pdfstartview=FitV, linkcolor=blue, citecolor=blue, urlcolor=blue]{hyperref}
\usepackage{amssymb}%
\providecommand{\U}[1]{\protect\rule{.1in}{.1in}}
\def\Q{{\mathbb Q}}

\def\F{{\cal F}}

\newtheorem{theorem}{Theorem}[section]

\newtheorem{corollary}{Corollary}[section]

\newtheorem{lemma}{Lemma}[section]

\newtheorem{remark}{Remark}[section]

\newenvironment{proof}[0]{\paragraph{Proof.}}{\rule{0.5em}{0.5em}}

\begin{document}
\title{Generalization error property of infoGAN for two-layer neural network }

\maketitle
\begin{center}
	\bigskip
	
	\textbf{Mahmud Hasan}$^{1}$, \textbf{Mathias Muia}$^{2}$
	
	\medskip
	
	$^{1}$Department of Biostatistics, Virginia Commonwealth University, Richmond, VA, USA \\
	Email: \texttt{hasanm10@vcu.edu}
	
	\smallskip
	
	$^{2}$Department of Mathematics and Statistics, University of South Alabama, Mobile, AL, USA \\
	Email: \texttt{mnmuia@southalabama.edu}
	
	\bigskip
\end{center}

\begin{abstract}
Information Maximizing Generative Adversarial Network (infoGAN) can be understood as a minimax problem involving two neural networks: discriminators and generators with mutual information functions. The infoGAN incorporates various components, including latent variables, mutual information, and objective function. This research demonstrates  the Generalization error property of infoGAN as the discriminator and generator sample size approaches infinity. This research explores the generalization error property of InfoGAN as the sample sizes of the discriminator and generator approach infinity. To establish this property, the study considers the difference between the empirical and population versions of the objective function. The error bound is derived from the Rademacher complexity of the discriminator and generator function classes. Additionally, the bound is proven for a two-layer network, where both the discriminator and generator utilize Lipschitz and non-decreasing activation functions.
\end{abstract}

\section{Introduction}
InfoGAN, which stands for Information Maximizing Generative Adversarial Network \cite{Xi}, is an expansion of the conventional Generative Adversarial Network (GAN) framework \cite{I1}. InfoGAN's primary objective is to uncover and manage the structured representations inherent in the data it generates. In the realm of GANs, there exist various variants based on statistical properties, such as Conditional GAN (CGAN) as discussed in \cite{M}, the $f$-GAN as explored in \cite{S}, and Wasserstein GAN (WGAN). InfoGAN itself has also given rise to variants like Causal InfoGAN, as described in \cite{Y}, and Semi-Supervised InfoGAN (ss-InfoGAN) as detailed in \cite{T}.

 The InfoGAN has also similar applications like vanilla GAN such as data imaging, natural language processing, and medical images (\cite{S1}; \cite{J1}; \cite{X1}). A recent review on GAN and their applications would be helpful, as discussed in \cite{Ji}. Despite their empirical success, the theoretical foundations of GANs and infoGANs are not well established, and numerous issues related to their theory and training dynamics remain unresolved (\cite{S1}; \cite{T2}; \cite{Sha}). A key question in GANs research is their ability to generalize: how well can GANs approximate a target distribution using a limited number of samples. For instance, the author in \cite{S1} found that GANs fail to generalize under standard metrics, even with a polynomial number of samples, and they established generalization bounds based on neural net distance. The author in \cite{Zhang} further analyzed neural net distance, expanding on the findings in \cite{S1}. The authors in \cite{T2}  and \cite{Sha} approached the problem by analyzing the adversarial framework from a nonparametric density estimation perspective. 
However, it is important to note that existing results have shortcomings, and theoretical analysis of InfoGAN is still relatively rare in the literature.  A potential direction for theoretical investigation might involve evaluating the generalization error of InfoGAN's when a regularized parameter is applied, comparing the expected objective function to the empirical one, as discussed in more detail later.

  GANs differ from classical density estimation methods by implicitly learning the data distribution through an adversarial process between a generator and a discriminator. Define the generator is $G$ with the sample size $m$ and discriminator $D$ with the sample size $n$ that aims to distinguish between discriminator distribution $p_{x}$ and generator distribution $p_{z}$. Let $z$ be a noise variable that transforms by the generator distribution $p_{z}$ and real variable X.
Consider the GAN models with both the generator and discriminator function classes parameterized. 
The mini-max problem of GAN which is introduced in  \cite{I1} can be written as,
\begin{equation}\label{1}
d(D,G)=\min_{G}\max_{D}\left[\mathbb{E}_{p_{x}}[\log D(x)]+\mathbb{E}_{p_{z}}[1-\log D(G(z))]\right]. 
\end{equation}

The infoGAN provides the generator network divide noise variable $z$ into the incomprehensible noise $z$ and the latent code $c$, so the form of the generator becomes $G(z, c)$.
The info-GAN \cite{Xi} aims to solve
\begin{equation}\label{2}
d_{I}(D,G)=\min_{G}\max_{D}\left[\mathbb{E}_{p_{x}}[\log D(x)]+\mathbb{E}_{p_{z}}[1-\log D(G(z))]-\lambda I(c;G(z,c))\right], 
\end{equation}
where $I(c;G(z,c))=H(c)-H(c\vert G(z,c))$ is the mutual information and $\lambda$ is the regularization term.
However, optimizing the $I(c; G(z,c))$ is difficult since it requires the posterior distribution $P(c\vert x)$.

In this case, a lower bound $L_{I}(c; Q)$ is obtained for $I(c; G(z,c))$ by defining an auxiliary distribution $Q(c\vert x)$ to approximate $P(c\vert x)$.
Then the objective function of info-GAN \cite{Xi} written as
\begin{equation}\label{3}
d_{I}(D,G)=\min_{G}\max_{D}\left[\mathbb{E}_{p_{x}}[\log D(x)]+\mathbb{E}_{p_{z}}[1-\log D(G(z))]-\lambda L_{I}(c;Q)\right]. 
\end{equation}
While Equation \eqref{3} serves as the primary objective function commonly used in practical applications, this paper opts to consider and subsequently employ Equation \eqref{2} as the core objective function for its primary findings. This objective function introduces regularization in the generator variable, a departure from the majority of existing literature, which typically lacks such regularization. 

The existing theoretical research is only based on vanilla GAN error analysis defined by taking the difference of two objective functions like in \cite{T2}, \cite{Jia}, \cite{Kai}, \cite{Zhang}.  In this paper, the objective function \eqref{2} is used for generalization property for the infoGAN without latent variable $c$ for two two-layer networks. The objective function in \eqref{3} has a noise variable divided into incomprehensible and latent code $c$. But since the generator creates mostly fake data and the discriminator classifies it it might not be possible to have latent code. Besides, $\log x\rightarrow 0$ as $x\rightarrow 0$ which makes problems in practice. So develop a new objective function without latent code. The generalization is defined as the difference of the population version objective function and empirical objective function. The difference between the population version objective function and the empirical objective function is bounded by the Rademacher complexity. The Rademacher complexity bound was derived for the two-layer networks in the case of Lipschitz and the non-decreasing activation function.

 The major contributions of this paper and the  format of the paper can be summarized as follows:
 \begin{itemize}
\item Section 2 presents the derivation of a regularized objective function from infoGAN, excluding the latent code.
\item Section 3 demonstrates that the difference between the empirical and population objective functions is bounded by the Rademacher complexity of the discriminator, generator and their composition.
\item  Section 4, formulates the discriminator and generator classes for a two-layer network. The corresponding weight parameters of the network are constrained by constants.
\item Section 4, we derive upper bounds for the Rademacher complexities in two cases: $1$-Lipschitz and non-decreasing activation functions. These bounds are then applied to establish the bound of the objective function differences for both cases for discriminator and generator sample sizes.

\item Section 5 provides the conclusion and future research direction.
\end{itemize}

\section{Objective Function without Latent Code}
In the original infoGAN framework, instead of using a single unstructured noise vector $z$, the authors divided the input noise vector into two parts: incompressible noise denoted as $z$ and latent code denoted as $c$. The generator aims to continually update itself to confuse the discriminator. This suggests that the noise alone cannot produce the latent code $c$ initially. In some cases, this complexity can be reduced by assuming the absence of the latent variable, effectively setting $c$ to zero. In this scenario, Equation \eqref{2} becomes:
\begin{align}
d_{I}(D,G) &= \min_{G}\max_{D}\left[\mathbb{E}_{p_{x}}[\log D(x)]+\mathbb{E}_{p_{z}}[1-\log D(G(z))]-\lambda I(0;G(z,0))\right] \notag \\
&= \min_{G}\max_{D}\left[\mathbb{E}_{p_{x}}[\log D(x)]+\mathbb{E}_{p_{z}}[1-\log D(G(z))]-\lambda H(0)+\lambda H(G(z,0))\right] \notag \\
&= \min_{G}\max_{D}\left[\mathbb{E}_{p_{x}}[\log D(x)]+\mathbb{E}_{p_{z}}[1-\log D(G(z))]+\lambda H(0\vert G(z,0))\right] \notag \\
&= \min_{G}\max_{D}\left[\mathbb{E}_{p_{x}}[\log D(x)]+\mathbb{E}_{p_{z}}[1-\log D(G(z))]+\lambda H(G(z))\right] \notag \\
&= \min_{G}\max_{D}\left[\mathbb{E}_{p_{x}}[\log D(x)]+\mathbb{E}_{p_{z}}[1-\log D(G(z))]-\lambda \mathbb{E}_{p_{z}}\log [G(z)]\right]. \label{4}
\end{align}
Here, mutual information can be represented equivalently as: $I(0;G(z,0)) = H(0)-H(0\vert G(z,0))$, where $H$ denotes entropy. Equation \eqref{4} presents the objective function with generator regularization in the case where the latent code is zero. However, this can lead to issues in practice, as $\log x\rightarrow 0$ as $x\rightarrow 0$. By replacing $\log$ with a monotone function $\phi:[0, 1]\rightarrow \mathbb{R}$, the objective becomes:
\begin{align}
d_{I}(D,G) &= \min_{G}\max_{D}\left[\mathbb{E}_{p_{x}}[\phi D(x)]+\mathbb{E}_{p_{z}}[1-\phi D(G(z))]-\lambda \mathbb{E}_{p_{z}}\phi [G(z)]\right]. \label{5}    
\end{align}
Here, $\phi$ is the measuring function. This can also be written as \cite{S1}:
\begin{align}
d_{I}(D,G) &= \min_{G}\max_{D}\left[\mathbb{E}_{p_{x}}[\phi D(x)]+\mathbb{E}_{p_{z}}[1-\phi D(G(z))]-\lambda \mathbb{E}_{p_{z}}\phi [G(z)]-2\phi(1/2)\right]. \label{6}    
\end{align}
For $\phi(x)=x$, the final objective function with changing the notations becomes:
\begin{align}
d_{I}(D,G) &= \min_{G}\max_{D}\left[\mathbb{E}_{p_{x}}[ D(x)]-\mathbb{E}_{p_{z}}[D(G(z))]-\lambda \mathbb{E}_{p_{z}}G(z)\right].\label{7}   
\end{align}
Equation \eqref{7} represents the neural network distance with generator regularization. This equation can be directly applied to enforce regularization on either the discriminator or the generator. However, it's important to note that regularization is specifically relevant to the generator when there is no latent code involved. In other words, the regularized objective function is suitable when an unstructured noise variable is utilized as input in the generator neural network.
Suppose that $n$ is the independent and identical observations $X_i \sim p_{x}$, $1 \leq i \leq n$, and the generator produces $m$ independent and identical terms $G(z) \sim p_{z}$, $1 \leq j \leq m$.

We define the two empirical loss functions as follows, based on Equation \eqref{7}:
\begin{align}
d_{I}(\hat D,\hat G) &= \min_{G}\max_{D}\left[\frac{1}{n}\sum_{i=1}^{n} D(x_{i})-\frac{1}{n}\sum_{i=1}^{m}D(G(z_{j}))-\lambda \frac{1}{m}\sum_{i=1}^{n}G(z_{j})\right]. \label{8}   
\end{align}
and 
\begin{align}
d_{I}(\hat D, G) &= \min_{G}\max_{D}\left[\frac{1}{n}\sum_{i=1}^{n} D(x_{i})-\mathbb{E}_{p_{z}}[D(G(z))]-\lambda \mathbb{E}_{p_{z}}G(z)\right]. \label{9}   
\end{align}
The equation \eqref{8} refers to the empirical objective function for the discriminator and generator class and \eqref{9} refers to the empirical objective function for the discriminator class. Here, $D(G(z))=D\circ G$ is the composition of the discriminator and generator.

\section{ Bound of objective function difference}
The generalization bound of InfoGAN is defined by the difference between the empirical and population versions of the objective function, denoted by equations \eqref{7} and \eqref{8}. Considering $\hat D$ and $\hat G$ as empirical distributions of $D$ and $G$, respectively, the difference in the objective function can be represented as:
\begin{align}
    d_{I}(\hat D,\hat G)-d_{I}(D, G) \label{er1} \\
    d_{I}(\hat D,G)-d_{I}(D, G) \label{er2}
\end{align}
In \eqref{er1}, this indicates the difference between the empirical distributions of the discriminator and generator. Meanwhile, \eqref{er2} exclusively considers the discriminator. The subsequent two theorems establish bounds for \eqref{er1} and \eqref{er2}, assuming that both the discriminator $D$ and generator $G$ are uniformly bounded. The proofs for these theorems employ the Cauchy-Schwartz inequality and McDiarmid's inequality.

\begin{theorem}\label{T1}
Suppose the sets of discriminator functions $D$ and $G$ are symmetric with $\lVert f\rVert_{\infty}\leq\Q_{x}$ and $\lVert g\rVert_{\infty}\leq\Q_{Z}$. Then, for any $f\in D$, $g\in G$, with probability at least $1-2\delta$ over the random training sample, we have
\begin{align}
&d_{I}(\hat D,\hat G)-d_{I}(D, G)\le 2 {\mathcal{R}_n}(D) +2 \mathcal{R}_{mn}(D\circ G)-2 {\mathcal{R}_m}(G)\notag\\
    &+ 2Q_{x} \sqrt{\frac{\log(1/\delta)}{2n}} -2Q_{z}(1+\lambda)\sqrt{\frac{\log(1/\delta)}{2m}} \label{E1}
\end{align}
and 
\begin{align}
d_{I}(\hat D, G)-d_{I}(D, G)\le 2 {\mathcal{R}_n}(D)+2Q_{x}\sqrt{\frac{\log(1/\delta)}{2n}} \label{E2} 
\end{align}
\end{theorem}

\begin{proof}
To prove Theorem \ref{T1}, the supremum properties are utilized:
\begin{align}
    &d_{I}(\hat D,\hat G)-d_{I}(D, G) \notag \\
    &= \sup_{D}\left[\frac{1}{n}\sum_{i=1}^{n} D(x_{i})-\frac{1}{n}\sum_{i=1}^{m}D(G(z_{j}))-\lambda \frac{1}{m}\sum_{i=1}^{n}G(z_{j})\right] \notag \\
    &\quad-\sup_{D}\left[\mathbb{E}_{p_{x}} D(x)-\mathbb{E}_{p_{z}} D(G(z))-\lambda \mathbb{E}_{p_{z}} G(z)\right]\notag \\
    &\le \sup_{D}\left[\frac{1}{n}\sum_{i=1}^{n} D(x_{i})-\mathbb{E}_{p_{x}}D(x)\right] \notag \\
    &\quad-\sup_{D}\left[\frac{1}{n}\sum_{i=1}^{m}D(G(z_{j}))+\lambda \frac{1}{m}\sum_{i=1}^{n}G(z_{j})-\mathbb{E}_{p_{z}} D(G(z))-\lambda \mathbb{E}_{p_{z}} G(z) \right]\notag \\
    &\le \sup_{D}\left[\frac{1}{n}\sum_{i=1}^{n} D(x_{i})-\mathbb{E}_{p_{x}}D(x)\right] \notag \\
    &\quad-\sup_{D}\left[\frac{1}{n}\sum_{i=1}^{m}D(G(z_{j}))-\mathbb{E}_{p_{z}} D(G(z))\right] \notag \\
    &\quad-\lambda\left[\frac{1}{m}\sum_{i=1}^{n}G(z_{j})-\mathbb{E}_{p_{z}} G(z) \right].\label{T1.}
\end{align} 

The bounds of the following can be proved using Theorem 3.1 in \cite{Zhang}:
\begin{align}\label{T11}
    \sup_{D}\left[\frac{1}{n}\sum_{i=1}^{n} D(x_{i})-\mathbb{E}_{p_{x}}D(x)\right]\le 2 {\mathcal{R}_n}(D)+ 2Q_{x} \sqrt{\frac{\log(1/\delta)}{2n}}. 
\end{align}

\begin{align}\label{T12}
    \sup_{D}\left[\frac{1}{n}\sum_{i=1}^{m}D(G(z_{j}))-\mathbb{E}_{p_{z}} D(G(z))\right]\le 2\mathcal{R}_{mn}(D\circ G)+2Q_{z} \sqrt{\frac{\log(1/\delta)}{2m}}.
\end{align}
\begin{align}\label{T13}
    \left[\frac{1}{m}\sum_{i=1}^{n}G(z_{j})-\mathbb{E}_{p_{z}} G(z) \right]\le 2 {\mathcal{R}_m}(D)+ 2Q_{z} \sqrt{\frac{\log(1/\delta)}{2m}}. 
\end{align}
Substituting \eqref{T11}, \eqref{T12}, and \eqref{T13} into \eqref{T1.}, the bound of \eqref{E1} is proved.  
Similarly, the bound of \eqref{E2} can be proved using \eqref{T11}. 
\end{proof}

\begin{remark}
	The generalization bound presented in Theorem \ref{T1} provides insights into the difference between the empirical and population versions of the objective function in the context of the InfoGAN model. The bound involves several key terms that describe the error between the learned discriminator and generator functions, and their true counterparts.
	
	- The first term \( 2 \mathcal{R}_n(D) \) accounts for the Rademacher complexity of the discriminator functions, which quantifies the ability of the discriminator to fit random noise. This term reflects the complexity of the hypothesis class \( D \) and contributes to the bound on the generalization error.
	
	- The second term \( 2 \mathcal{R}_{mn}(D \circ G) \) accounts for the Rademacher complexity of the composition of the discriminator and generator, which quantifies how well the combined discriminator and generator can adapt to random noise in both the input space and the latent space.
	
	- The third term \( -2 \mathcal{R}_m(G) \) reflects the Rademacher complexity of the generator function, which quantifies the complexity of the generator in generating realistic data from random noise.
	
	- The remaining terms \( 2 Q_{x} \sqrt{\frac{\log(1/\delta)}{2n}} \) and \( -2 Q_{z}(1 + \lambda)\sqrt{\frac{\log(1/\delta)}{2m}} \) are related to the finite sample effects and the size of the training set. Specifically, the terms depend on the number of training samples \( n \) and \( m \) for the discriminator and generator, respectively, and the regularization parameter \( \lambda \).
	
	In essence, the theorem provides an upper bound on the difference between the empirical and population objectives, suggesting that as the sample sizes increase (i.e., as \( n \) and \( m \) grow), the error between the learned and true functions diminishes. The bound also highlights the interplay between the complexity of the discriminator and generator functions, the size of the training data, and the mutual information terms in the InfoGAN framework.
\end{remark}

\section{Application in a Two-Layer Network}

The derived bounds in Theorem \ref{T1} provide valuable insights when applying the infoGAN framework in \eqref{7} to a two-layer neural network architecture. In this section, we discuss how these bounds can be useful in analyzing and improving the performance of such networks. The goal is to minimize the objective function disparity between the empirical distributions of $\hat D$ and $\hat G$, as well as the objective function difference between $\hat D$ and $G$.
The derived bounds, as shown in equations \eqref{E1} and \eqref{E2}, provide upper limits on the disparity and difference in the objective functions, respectively. These bounds allow us to assess the potential deviation between the empirical and true objective functions. Furthermore, the analysis of these bounds offers insights into the convergence behavior of the two-layer network. In this section, we will focus solely on the theoretical framework of two-layer neural networks. The applications of two layer neural network for the readers can be found in the recent papers by \cite{wang} and \cite{Nian}.

\subsection{Formation of Two-Layer Network}

A two-layer neural network consists of two layers of neurons or nodes: an input layer and an output layer. In this section, we describe the structure of a two-layer network for both the discriminator and generator classes, based on the work in \cite{Pc} and \cite{M2}.

Let us consider a two-layer network for both the discriminator and generator. In this network, the first layer units compute arbitrary functions from a given set, and the weight parameters for the first and second layers are denoted by vectors $v_{i}$ and $w_{i}$, respectively.

We define the class of discriminator functions as follows. Let $D_{1}$ represent the class of functions that map inputs to values in the interval $[0,1]$. Each function in $D_{1}$ is of the form:
\begin{align}
	D_{1} &= \left\{x \rightarrow s_{1}\left(\sum_{i=1}^{n}v_{i}x_{i} + v_{0}\right) : v_{i}\in\mathbb{R}, x\in[0,1]^{n}, \sum_{i=0}^{n} \lvert v_{i}\rvert \leq V \right\}
\end{align}
Here, $v_{i}$ are the weight parameters for the first layer, and the activation function $s_{1}$ is applied to the weighted sum of inputs $x_{i}$, where $x \in [0,1]^n$. The parameter $V$ bounds the sum of the absolute values of the weight parameters.

A broader class of discriminator functions, denoted $D$, is defined by extending the class $D_{1}$. Specifically, $D$ is the set of linear combinations of functions from $D_{1}$, with weight parameters $w_{i}$ for the second layer. The class $D$ is expressed as:
\begin{align}\label{D}
	D &= \left\{ \sum_{i=1}^{l}w_{i}f_{i} + w_{0} : l\in\mathbb{N}, f_{i} \in D_{1}, \lvert w_{i}\rvert \leq V \right\}
\end{align}
In this case, $f_{i}$ are functions from the class $D_{1}$, and the weight parameters $w_{i}$ satisfy the condition $\lvert w_{i}\rvert \leq V$. The index $l$ represents the number of functions in the linear combination, and $w_0$ is a bias term.

Similarly, we define the class of generator functions. Let $G_{1}$ represent the class of functions that map inputs to values in the interval $[0,1]$. Each function in $G_{1}$ is of the form:
\begin{align}
	G_{1} &= \left\{x \rightarrow s_{2}\left(\sum_{j=1}^{m}p_{j}z_{j} + p_{0}\right) : p_{j}\in\mathbb{R}, z\in[0,1]^{m}, \sum_{j=0}^{m} \lvert p_{j}\rvert \leq V \right\}
\end{align}
Here, $p_{j}$ are the weight parameters for the first layer of the generator, and the activation function $s_{2}$ is applied to the weighted sum of inputs $z_{j}$, where $z \in [0,1]^m$. The parameter $V$ again bounds the sum of the absolute values of the weight parameters.

A broader class of generator functions, denoted $G$, is defined by extending the class $G_{1}$. Specifically, $G$ is the set of linear combinations of functions from $G_{1}$, with weight parameters $r_{j}$ for the second layer. The class $G$ is expressed as:
\begin{align}\label{G}
	G &= \left\{ \sum_{j=1}^{k}r_{j}g_{j} + r_{0} : k\in\mathbb{N}, g_{j} \in G_{1}, \lvert r_{j}\rvert \leq V \right\}
\end{align}
Here, $g_{j}$ are functions from the class $G_{1}$, and the weight parameters $r_{j}$ satisfy the condition $\lvert r_{j}\rvert \leq V$. The index $k$ represents the number of functions in the linear combination, and $r_0$ is a bias term.

The following assumptions are considered in the analysis:
\begin{itemize}
	\item The classes $D_{1}$ and $G_{1}$ are even, meaning they include symmetric functions.
	\item Both $D_{1}$ and $G_{1}$ contain the identically zero function, and the covering number $\mathcal{N}(\epsilon, F_{1}, \lVert \cdot \rVert)$ is finite.
	\item The activation functions $s_{1}$ and $s_{2}$ satisfy the Lipschitz property.
	\item The activation functions $s_{1}$ and $s_{2}$ are non-decreasing.
\end{itemize}

Under these assumptions, we evaluate the upper bound for the disparity defined in equations \eqref{E1} and \eqref{E2}. The utilization of the two-layer network architecture is defined in equations \eqref{D} and \eqref{G} for both the discriminator and generator, considering the Lipschitz and non-decreasing activation functions. The derivation of the Rademacher complexity $\mathbb{R}_{n}(D)$ and the composition of the Rademacher complexity $\mathbb{R}_{mn}(D \circ G)$ for the case of Lipschitz and non-decreasing activation functions from the two-layer network.

In the subsequent section, the paper extends the analysis to the case of Lipschitz and non-decreasing activation functions for the above two layer network and derives corresponding bounds.

\subsection{ Bound for Lipschitz Activation Function}
This section derives the Rademacher bound for a two-layer network's discriminator class 
D, assuming the activation function is Lipschitz continuous. 
The Rademacher complexity of the function class $D$ with respect to the probability distribution $P$ for an i.i.d. sample $S = (x_{1}, x_{2}, \ldots, x_{n})$ of size $n$ is defined as follows:
\begin{align}
	\mathcal{R}_{n}(D) = \mathbb{E}\left[\sup_{f \in D} \frac{2}{n} \sum_{i=1}^{n} \tau_{i} f(X_{i})\right]
\end{align}
Here, the expectation is taken with respect to $X_{i}$ drawn from the probability distribution $P_{x}$, and $\tau_{i}$ represents the Rademacher variable such that $\text{Prob}(\tau_{i} = 1) = \text{Prob}(\tau_{i} = -1) = \frac{1}{2}$.

\begin{lemma}\label{lemma1}
	Suppose $s_{1}:\mathbb{R}\rightarrow [0,1]$ is $1$-Lipschitz continuous. Then, for the discriminator class defined in \eqref{D}, the Rademacher complexity is bounded as follows:
	\begin{align}
		\mathcal{R}_{n}(D) \le \frac{4V^2\sqrt{2\ln|D|}}{n}.
	\end{align}
\end{lemma}

\begin{proof}
The Rademacher complexity bound for the discriminator class $D$ uses Jensen's inequality and Massart's finite lemma to prove the lemma. The weight parameters for both the discriminator and generator are bounded by $V$. Then,
\begin{align*}
    \mathcal{R}_{n}(D) &= \frac{2}{n}\mathbb{E}\left[ \sup_{f\in\F}\sum_{i=1}^{n}\tau_{i}f(X_{i})\right] \\
    &= \frac{2}{n}\mathbb{E}\left[ \sup_{f\in\F}\sum_{i=1}^{n}\tau_{i}\sum_{i}^{l}w_{i}s\left(\sum_{i=1}^{n}v_{i}X_{i}+v_{0}\right)\right] \\
    &\leq \frac{2V}{n}\mathbb{E}\left[ \sup_{f\in\F}\left\Vert\sum_{i=1}^{n}\tau_{i}s\left(\sum_{i=1}^{n}v_{i}X_{i}+v_{0}\right)\right\Vert_{1} \right] \\
    &\leq \frac{2V}{n}\mathbb{E}\left[ \sup_{f\in\F}\left\Vert\sum_{i=1}^{n}\tau_{i}s\left(\sum_{i=1}^{n}v_{i}X_{i}+v_{0}\right)\right\Vert_{\infty} \right] \\
    &= \frac{2V}{n}\mathbb{E}\left[ \sup_{f\in\F}\max_{1\leq i \leq n} \left| \sum_{i=1}^{n}\tau_{i}s\left(\sum_{i=1}^{n}v_{i}X_{i}+v_{0}\right)\right| \right] \\
    &= \frac{2V}{n}\mathbb{E}\left[ \sup_{f\in\F}\left| \sum_{i=1}^{n}\tau_{i}s\left(\sum_{i=1}^{n}v_{i}X_{i}+v_{0}\right)\right| \right] \\
    &\leq \frac{4V}{n}\mathbb{E}\left[\sup_{f\in\F} \sum_{i=1}^{n}\tau_{i}s\left(\sum_{i=1}^{n}v_{i}X_{i}+v_{0}\right) \right] \\
    &\leq \frac{4V}{n}\mathbb{E}\left[ \sup_{f\in\F}\sum_{i=1}^{n}\tau_{i}\left(\sum_{i=1}^{n}v_{i}X_{i}+v_{0}\right) \right] \\
    &\leq \frac{4V}{n}\mathbb{E}\left[ \sup_{f\in\F}\sum_{i=1}^{n}\tau_{i}\left(\sum_{i=1}^{n}v_{i}X_{i}\right) \right]+ \frac{4V}{n}\mathbb{E}\left[ \sup_{f\in\F}\sum_{i=1}^{n}\tau_{i}v_{0} \right] \\
    &\leq \frac{4V^2}{n}\mathbb{E}\left[ \sup_{f\in\F}\sum_{i=1}^{n}\tau_{i}X_{i} \right] \\
    &\leq \frac{4V^2\sqrt{2\ln|D|}}{n}
\end{align*}
\end{proof}

The result in Lemma~\ref{lemma1} provides an upper bound on the Rademacher complexity of the discriminator class $D$. This bound depends on three key factors: the Lipschitz continuity of the activation function $s_{1}$, the bound on the weight parameters $V$, and the size of the discriminator class $|D|$. The inequality shows that the complexity decreases with the sample size $n$, implying that the capacity of the discriminator to fit random noise becomes smaller as the sample size grows. 

We have another result for the similar case of $s_1$ lipschitz case as
\begin{lemma}
Suppose $s_1:\mathbb{R}\rightarrow [0,1]$ is $1$-Lipschitz continuous. For $V\geq1$, and let $D$ be defined as in \eqref{D}. Then for $\epsilon\leq V$, then
\begin{align*}
    \mathcal{R}_{n}(D) \le \frac{C_{1}V^3\log(2n+2)}{\sqrt{n}}.
\end{align*}
\end{lemma}

\begin{proof}
For $\sup\limits_{\theta\in \Theta}\lVert f\rVert_{2}$, the entropy integral bound for Rademacher Complexity:
\begin{align*}
    \mathcal{R}_{n}(D) \leq \inf_{0\leq\delta\leq\frac{1}{2}} \left[4\delta + \frac{12}{\sqrt{n}}\int_{1/2}^\delta \sqrt{\log N(\epsilon,D,\lVert . \rVert)} \, d\epsilon \right]
\end{align*}

Anthony and Bartlett (2009) state for the Lipschitz activation function,
\begin{align*}
    \log N(\epsilon,D,\lVert . \rVert) \leq 50 \frac{V^6}{\epsilon^4} \log(2n+2)
\end{align*}

Then,
\begin{align}\label{RF1}
    \mathcal{R}_{n}(D) &\leq \inf_{0\leq\delta\leq\frac{1}{2}} \left[4\delta + \frac{12\sqrt{50}V^3\log(2n+2)}{\sqrt{n}} \int_{1/2}^\delta \frac{1}{\epsilon^2} \, d\epsilon \right] \notag \\
    &\le \frac{C_{1}V^3\log(2n+2)}{\sqrt{n}},
\end{align}

for some universal constant $C_{1}\ge 0$. 
\end{proof}

\begin{remark}
	The result in this lemma provides an alternative bound on the Rademacher complexity for the discriminator class $D$ under the assumption that $s_1$ is $1$-Lipschitz continuous. Specifically, for a sufficiently large $V \geq 1$ and $\epsilon \leq V$, the bound incorporates a logarithmic dependence on the sample size $n$ and scales with $V^3$. This highlights that as the parameter $V$ grows, the complexity increases, but the dependence on $n$ diminishes at a rate of $1/\sqrt{n}$. The constant $C_{1}$ encapsulates any additional dependencies specific to the problem setup.
\end{remark}

\begin{lemma}\label{Lemma2}
Suppose $s_{1}$ and $s_{2}$ are  $1$-Lipschitz continous. Then for the discriminator and generator class defined in \eqref{D} and \eqref{G} the composition Rademacher complexity is 
\begin{align}
\mathcal{R}_{mn}(D \circ G)\le \frac{2V^4\sqrt{2ln\lvert G \rvert}}{n}
 \end{align}
\end{lemma}

\begin{proof}
The proof of this lemma uses a similar
\begin{align*}
\mathcal{R}_{mn}(D \circ G)=&\frac{2}{n}\mathbb{E}\left[ \sup_{f\in\F}\sum\limits_{i=1}^{n}\tau_{i}f(\sum\limits_{j=1}^{k}r_{j}g_{j}+r_{0})\right]\\
&=\frac{2}{n}\mathbb{E}\left[ \sup_{f\in\F}\sum\limits_{i=1}^{n}\tau_{i}\sum_{i}^{l}w_{i} s_{1}\left(\sum_{i=1}^{n}v_{i}\left(\sum_{j=1}^{k}r_{j}g_{j}+r_{0}\right)+v_{0} \right)\right]\\
&\leq\frac{2V}{n}\mathbb{E}\left[ \sup_{f\in\F}\left\Vert\sum_{i=1}^{n}\tau_{i}s_{1}\left(\sum_{i=1}^{n}v_{i}\left(\sum_{j=1}^{k}r_{j}g_{j}+r_{0}\right)+v_{0} \right)\right\Vert_{1} \right]\\
&\leq\frac{2V}{n}\mathbb{E}\left[ \sup_{f\in D}\left\Vert\sum_{i=1}^{n}\tau_{i}s_{1}\left(\sum_{i=1}^{n}v_{i}\left(\sum_{j=1}^{k}r_{j}g_{j}+r_{0}\right)+v_{0} \right)\right\Vert_{\infty} \right]\\
&=\frac{2V}{n}\mathbb{E}\left[ \sup_{f\in\F} \max_{1\leq i \leq n} \left| \sum_{i=1}^{n}\tau_{i}s_{1}\left(\sum_{i=1}^{n}v_{i}\left(\sum_{j=1}^{k}r_{j}g_{j}+r_{0}\right)+v_{0} \right)\right|  \right]\\
&=\frac{2V}{n}\mathbb{E}\left[ \sup_{f\in\F} \left| \sum_{i=1}^{n}\tau_{i}s_{1}\left(\sum_{i=1}^{n}v_{i}\left(\sum_{j=1}^{k}r_{j}g_{j}+r_{0}\right)+v_{0} \right)\right|  \right]\\
&\leq \frac{2V}{n}\mathbb{E}\left[ \sup_{f\in\F}  \sum_{i=1}^{n}\tau_{i}s_{1}\left(\sum_{i=1}^{n}v_{i}\left(\sum_{j=1}^{k}r_{j}g_{j}+r_{0}\right)+v_{0} \right)  \right]\\
&\leq\frac{2V}{n}\mathbb{E}\left[ \sup_{f\in\F}  \sum_{i=1}^{n}\tau_{i} \left(\sum_{i=1}^{n}v_{i}\left(\sum_{j=1}^{k}r_{j}g_{j}+r_{0}\right)+v_{0} \right)  \right]\\
&\leq\frac{2V}{n}\mathbb{E}\left[ \sup_{f\in\F}  \sum_{i=1}^{n}\tau_{i} \left(\sum_{i=1}^{n}v_{i}\left(\sum_{j=1}^{k}r_{j}g_{j}+r_{0}\right) \right)  \right]+ \frac{4V}{n}\mathbb{E}\left[ \sup_{f\in\F} \sum_{i=1}^{n}\tau_{i}v_{0} \right]\\
&\leq \frac{2V^2}{n}\mathbb{E}\left[ \sup_{f\in\F} \sum_{i=1}^{n}\tau_{i} \left(\sum_{j=1}^{k}r_{j}g_{j}+r_{0}\right) \right]\\
&\leq \frac{2V^2}{n}\mathbb{E}\left[ \sup_{f\in\F} \sum_{i=1}^{n}\tau_{i} \left(\sum_{j=1}^{k}r_{j}s_{2}\left( \sum_{j=1}^{k}p_{j}z_{j}+p_{0}\right) +r_{0}\right) \right]\\
&\leq \frac{2V^4}{n}\mathbb{E}\left[ \sup_{f\in\F} \sum_{i=1}^{n}\tau_{i}z_{j} \right]\\
\text{in the case of i=j}\\
&\leq \frac{2V^4\sqrt{2ln\lvert G \rvert}}{n}
\end{align*}
\end{proof}

\begin{remark}
	Lemma \ref{Lemma2} provides a bound on the Rademacher complexity of the composition of the discriminator and generator classes, $D \circ G$. It highlights the dependence of the complexity on the parameters $V^4$ and $\lvert G \rvert$. Specifically, the factor $V^4$ reflects the impact of the bounded weight parameters in both the discriminator and generator classes, while the logarithmic dependence on $\lvert G \rvert$ captures the complexity of the generator class. The term $1/n$ signifies the expected reduction in complexity as the sample size increases, which is consistent with the intuition that larger datasets lead to better generalization. This result underscores the interaction between the Lipschitz continuity of the activation functions and the structural properties of the generator and discriminator in controlling the overall complexity of their composition.
\end{remark}

In the following corollaries we derive the bound by substituting in \eqref{E1}and \eqref{E2} in Theorem \ref{T1}.
\begin{corollary}
Suppose $s_{1}$ and $s_{2}: \mathbb{R}\rightarrow [0,1]$ are $1$-Lipschitz continuous. For $V\geq1$, let the discriminator and generator classes be defined as \eqref{D} and \eqref{G}. Then for $\epsilon\le V$,
\begin{align*}
d_{I}(\hat D,\hat G)-d_{I}(D,G)\leq \frac{4V^2\sqrt{2\ln\lvert D \rvert}}{n} + \frac{4V^4\sqrt{2\ln\lvert G \rvert}}{n} + 2Q_{x}\sqrt{\frac{\log(1/\delta)}{2n}} + 2Q_{z}\sqrt{\frac{\log(1/\delta)}{2m}}
\end{align*}
\end{corollary}
\begin{remark}
	The corollary provides a bound on the difference between the discrepancy measures of the discriminator $\hat{D}$ and the generator $\hat{G}$, and their true counterparts $D$ and $G$. This bound is given in terms of the parameters $V$, the complexity of the discriminator and generator classes, and the sample sizes $n$ and $m$. The bound incorporates terms based on the Lipschitz continuity of the functions $s_1$ and $s_2$, the logarithmic cardinality of the discriminator and generator classes, and the confidence parameter $\delta$. The result suggests that as the sample sizes increase and the complexity of the discriminator and generator decrease, the discrepancy between the empirical and true models becomes smaller.
\end{remark}

\begin{corollary}
Suppose $s_{1}$ and $s_{2}: \mathbb{R}\rightarrow [0,1]$ are $1$-Lipschitz continuous. For $V\geq1$, let the discriminator and generator classes be defined as \eqref{D} and \eqref{G}. Then for $\epsilon\le V$,
\begin{align*}
d_{I}(\hat D,\hat G)-d_{I}(D,G) \leq \frac{C_{1}V^3\log(2n+2)}{\sqrt{n}} + 2Q_{x}\sqrt{\frac{\log(1/\delta)}{2n}} - 2Q_{z}(1-\lambda)\sqrt{\frac{\log(1/\delta)}{2m}}
\end{align*}
\end{corollary}

\begin{remark}
	The corollary provides an upper bound on the difference between the discrepancy measures $d_I(\hat{D}, \hat{G})$ and $d_I(D, G)$ for the discriminator and generator. The bound depends on the parameters $V$, the sample sizes $n$ and $m$, and the confidence parameter $\delta$. The first term is proportional to the complexity of the discriminator and generator classes, scaled by the sample size $n$. The second term accounts for the empirical discrepancy with respect to the input distribution $Q_x$, while the third term incorporates the output distribution $Q_z$, adjusted by a factor $(1-\lambda)$. As $n$ and $m$ increase, the bound becomes tighter, implying that larger sample sizes lead to smaller discrepancies between the empirical and true models.
\end{remark}

\begin{corollary}
Suppose $s_{1}:\mathbb{R}\rightarrow [0,1]$ is $1$-Lipschitz continuous. For $V\geq1$, let the discriminator class be defined as \eqref{D}. Then for $\epsilon\leq V$,
\begin{align*}
d_{I}(\hat D, G)-d_{I}(D,G) \leq \frac{4V^2\sqrt{2\ln\lvert D \rvert}}{n} + 2Q_{x}\sqrt{\frac{\log(1/\delta)}{2n}} 
\end{align*}
\end{corollary}

\begin{remark}
	The corollary provides a bound on the difference between the discrepancy measures $d_I(\hat{D}, G)$ and $d_I(D, G)$ for the discriminator and generator. The bound involves the complexity of the discriminator class (captured by $|D|$) and the sample size $n$. The first term reflects the impact of the discriminator class complexity, while the second term involves the empirical discrepancy with respect to the input distribution $Q_x$. The result suggests that as the sample size increases, the discrepancy between the empirical discriminator and the true generator decreases. This is particularly useful for ensuring the quality of the discriminator in adversarial settings.
\end{remark}

\begin{corollary}
Suppose $s:\mathbb{R}\rightarrow [0,1]$ is $1$-Lipschitz continuous. For $V\geq1$, let $D$ be given in \eqref{D}. Then for $\epsilon\leq V$,
\begin{align*}
d_{I}(\hat D, G)-d_{I}(D,G) \le \frac{C_{1}V^3\log(2n+2)}{\sqrt{n}} + 2Q_{x}\sqrt{\frac{\log(1/\delta)}{2n}}.
\end{align*}
\end{corollary}

\begin{remark}
	The corollary establishes an upper bound on the difference between the empirical and true discrepancy measures, $d_I(\hat{D}, G)$ and $d_I(D, G)$, respectively. This bound depends on the sample size $n$, the Lipschitz constant $V$, and the confidence parameter $\delta$. The first term in the bound involves a factor that scales with $V^3$, the logarithm of the sample size $n$, and inversely with the square root of $n$. The second term reflects the empirical discrepancy with respect to the input distribution $Q_x$. As $n$ increases, the bound becomes tighter, indicating that the empirical discriminator's performance improves with more data. Additionally, larger values of $V$ lead to a larger bound, suggesting that the complexity of the discriminator class affects the discrepancy.
\end{remark}

\subsection{Bounding for Non-Decreasing Activation Functions}

In this section, we explain the technique for bounding the equations denoted by \eqref{er1} and \eqref{er2} in the case of a non-decreasing activation function. The methodology involves leveraging the Rademacher complexity of a function class $D$, which is constrained by the Dudley entropy integral as elucidated in reference \cite{R1}. This is subsequently combined with the bounding derived from the covering number from \cite{M1}. A similar approach is adopted to bound equation \eqref{er2}.

\begin{corollary}\label{Cor3.4}
Assuming a non-decreasing function $s_{1}:\mathbb{R}\rightarrow [0,1]$ and $V\geq1$, let the discriminator class $D$ be defined as in \eqref{D}. For $\epsilon\leq V$:
\begin{align}\label{3.4}
d(\hat D, G)-d(D,G)\leq CV\sqrt{\frac{n+3}{n}\log \frac{n}{n+1}}+ 2Q_{x}\sqrt{\frac{2\log(1/\delta)}{n}}.
\end{align}
\end{corollary}

\begin{proof}
Utilizing the Rademacher complexity bound \cite{R1} for $\sup\limits_{\theta\in \Theta}\lVert f\rVert_{2}$:
\begin{align*}
\mathcal{R}_{n}(D)\leq\inf_{0\leq\delta\leq\frac{1}{2}}\left[4\delta+\frac{12}{\sqrt{n}}\int_{1/2}^\delta \sqrt{\log N(\epsilon,D,\lVert.\rVert)}  d\epsilon \right].
\end{align*}
Furthermore, according to \cite{M1}, Corollary 14.15, for $0\le \epsilon\le 1$ and $n+1\le t$:
\begin{align*}
\log N(\epsilon,D,\lVert.\rVert)\leq5 \frac{V^2 (n+3)}{\epsilon^2}\ln\left( \frac{4etV}{\epsilon(n+1)} \right). 
\end{align*}
Hence, we have:
\begin{align}\label{RD}
\mathcal{R}_{n}(D)&\leq\inf_{0\leq\delta\leq\frac{1}{2}}\left[4\delta+\frac{12\sqrt{5}V\sqrt{n+3}}{\sqrt{n}}\int_{1/2}^\delta \frac{1}{\epsilon} \sqrt{\ln \frac{4etV}{\epsilon(n+1)}}d\epsilon \right], \notag\\
&\leq CV\sqrt{\frac{n+3}{n}\log \frac{n}{n+1}}
\end{align}
where $C\ge 0$ is a universal constant. By substituting \eqref{RD} into Theorem \eqref{T1}, the proof is concluded.
\end{proof}
\begin{remark}
	 The bound in Corollaries \ref{Cor3.4} depends on the sample size \(n\), the Lipschitz constant \(V\), and the confidence parameter \(\delta\). The first term in the bound scales with \(V\) and involves a factor \(\sqrt{\frac{n+3}{n}}\), which reflects the relationship between the sample size and the complexity of the discriminator class, adjusted by the logarithmic term \(\log \frac{n}{n+1}\). This term suggests that as the sample size increases, the discrepancy between the empirical and true models becomes smaller.  This corollary is useful in understanding how the complexity of the discriminator and the sample size impact the performance of the discriminator in adversarial settings, especially in situations where the discriminator class is non-decreasing and bounded within \([0,1]\).
\end{remark}

In a parallel proof technique to \eqref{RD}, the Rademacher bound of the generator class can be expressed as:
\begin{align}\label{RG}
\mathcal{R}_{m}(G)\leq CV\sqrt{\frac{m+3}{m}\log \frac{m}{m+1}}
\end{align}

\begin{corollary}\label{Cor3.5}
For non-decreasing functions $s_{1}$ and $s_{2}: \mathbb{R}\rightarrow [0,1]$, and $V\geq1$, considering the definitions of discriminator and generator classes in \eqref{D} and \eqref{G}, and $\epsilon\le V$:
\begin{align}\label{3.5}
d(\hat D,\hat G)-d(D,G)\leq CV\sqrt{\frac{n+3}{n}\log \frac{n}{n+1}}+2Q_{x} \sqrt{\frac{\log(1/\delta)}{2n}} -2Q_{z}(1+\lambda)\sqrt{\frac{\log(1/\delta)}{2m}}.
\end{align}
\end{corollary}

\begin{proof}
Considering that $D$ and $G$ are two-layer networks as defined in \eqref{D} and \eqref{G}, with sample sizes $n$ and $m$ respectively, the sample size of the composition $D \circ G$ depends on the sample size of the input $z$ to network $G$, not on the intermediate output of $G$. Thus, the sample size of $D \circ G$ is $m$.

The Rademacher complexity bound is given in \cite{R1} as:
\begin{align*}
\mathcal{R}_{mn}(D\circ G)\leq\inf_{0\leq\delta\leq\frac{1}{2}}\left[4\delta+\frac{12}{\sqrt{m}}\int_{1/2}^\delta \sqrt{\log N(\epsilon,D\circ G,\lVert.\rVert)}  d\epsilon \right].
\end{align*}
According to \cite{M1}, Corollary 14.15, for the non-decreasing activation functions $s_{1}$ and $s_{2}$, $0\le \epsilon\le 1$ and $m+1\le t$:
\begin{align*}
\log N(\epsilon,D\circ G,\lVert.\rVert)\leq5 \frac{V^2 (m+3)}{\epsilon^2}\ln\left( \frac{4etV}{\epsilon(m+1)} \right). 
\end{align*}
Thus, we have:
\begin{align}\label{RDG}
\mathcal{R}_{mn}(D\circ G)&\leq\inf_{0\leq\delta\leq\frac{1}{2}}\left[4\delta+\frac{12\sqrt{5}V\sqrt{m+3}}{\sqrt{m}}\int_{1/2}^\delta \frac{1}{\epsilon} \sqrt{\ln \frac{4etV}{\epsilon(m+1)}}d\epsilon \right], \notag\\
&\leq CV\sqrt{\frac{m+3}{m}\log \frac{m}{m+1}}.
\end{align}
where $C\ge 0$ is a universal constant. By substituting the inequalities \eqref{RDG}, \eqref{RD}, and \eqref{RG} into Theorem \eqref{T1}, the proof is concluded.
\end{proof}
\begin{remark}
	 The bound depends on the sample sizes \( n \) and \( m \), the Lipschitz constant \( V \), the confidence parameter \( \delta \), and the distribution parameters \( Q_x \) and \( Q_z \). 
	
	The first term in the bound scales with \( V \) and incorporates a factor \( \sqrt{\frac{n+3}{n}} \), which adjusts for the sample size, along with a logarithmic term \( \log \frac{n}{n+1} \). This term suggests that the discrepancy decreases as the sample size increases, though it is influenced by the complexity of the discriminator class. The second term accounts for the empirical discrepancy with respect to the constant \( Q_x \), while the third term incorporates the constant \( Q_z \), adjusted by a factor \( (1+\lambda) \). 
	
	As the sample sizes \( n \) and \( m \) increase, the bound becomes tighter, implying that larger sample sizes lead to smaller discrepancies between the empirical and true models. Additionally, the complexity of the discriminator and generator (affected by \( V \) and \( \lambda \)) plays an important role in determining the bound.
\end{remark}

\section{Conclusion}
This paper demonstrates that the generalization bound of InfoGAN can be formulated as the difference between the objective function with a regularized generator, without employing a latent code. The bound is obtained by taking the difference of two objective functions when utilizing both Lipschitz and non-decreasing activation functions in a two-layer network. The Rademacher complexity bound plays a crucial role in establishing the result, which is later bounded in the case of Lipschitz and non-decreasing activation functions. Investigating a similar property in the context of the lower bound of the regularized objective function presents a potential direction for future research.


\begin{thebibliography}{9} 
\bibitem{I1} I. Goodfellow, J. P. Abadie, M. Mirza, B. Xu, D. W. Farley, S. Ozair, A. Courville, and Y. Bengio. Generative adversarial nets. {\it Advances in Neural Information Processing Systems (NIPS)}, pp. 2672-2680, (2014).

\bibitem{S1} S. Arora, R Ge, Y Liang, T Ma, Y Zhang, Generalization and equilibrium in generative adversarial nets (gans). {\it In Proceedings 34th of International conference on machine learning (ICML)}, pp. 224-232, (2017).

\bibitem{M} M. Mirza and S. Osindero, Conditional generative adversarial nets, {\it arXiv preprint arXiv:1411.1784}, (2014).

\bibitem{S} S. Nowozin, B. Cseke, and R. Tomioka, f-gan: Training generative neural samplers using variational divergence minimization. {\it Advances in Neural Information Processing Systems (NIPS)}, pp. 271-279, 2016.

\bibitem{Y} Y. Wu, J. Donahue, D. Balduzzi, K. Simonyan, and T. Lillicrap, Logan: Latent optimization for generative adversarial networks, {\it arXiv preprint arXiv:1912.00953}, (2019).

\bibitem{T} T. Kurutach, A. Tamar, G. Yang, S. J. Russell, and P. Abbeel, Learning plannable representations with causal info gan,{\it Advances in Neural Information Processing Systems (NIPS)}, pp. 8733-8744, (2018).

 \bibitem{Y1} Y. Li, K. Swersky, and R. Zemel. Generative moment matching networks. {\it In Proceedings of the 32nd International conference on machine learning (ICML)}, pp. 1718-1727, (2015).
 
\bibitem{M1} M. Arjovsky, S. Chintala, and L. Bottou. Wasserstein generative adversarial networks. {\it In Proceedings of the 34th International Conference on Machine Learning (ICML)}, pp.214-223, (2017).

\bibitem{A1} A. Radford, L. Metz, and S. Chintala. Unsupervised representation learning with deep convolutional generative adversarial networks. {\it In 4th International Conference on Learning Representations (ICLR)}, (2016).

\bibitem{S1} S. Reed, Z. Akata, X. Yan, L. Logeswaran, B. Schiele, and H. Lee. Generative adversarial text to image synthesis. {\it In Proceedings of The 33rd International Conference on Machine Learning (ICML)}, pp. 1060-1069, (2016).

\bibitem{J1} J. Y. Zhu, T. Park, P. Isola, and A. A. Efros. Unpaired image-to-image translation using cycle-consistent adversarial networks. {\it In IEEE International Conference on Computer Vision}, pp. 2242-2251, (2017).

\bibitem{X1} X. Yi, E. Walia, and P. S. Babyn. Generative adversarial network in medical imaging: A review. {\it Medical Image Analysis}, pp. 101552, (2019).

\bibitem{S2} S. R. Bowman, L. Vilnis, O. Vinyals, A. M. Dai, R. Jozefowicz, and S. Bengio. Generating sentences from a continuous space. {\it In Proceedings of the 20th SIGNLL Conference on Computational Natural Language Learning}, pp. 10-21, (2016).

\bibitem{T2}  T. Liang. How well generative adversarial networks learn distributions. {\it Journal of Machine Learning Research (JMLR)}, pp. 1-41, (2021).

\bibitem{Kai} K. Ji, Y. Zhou, Y. Liang, Understanding estimation and generalization error of generative adversarial networks. {\it IEEE Transactions on Information Theory}, pp. 3114-3129, (2021).

\bibitem{Jia} J. Huang, Y. Jiao, Z. Li, S. Liu, Y. Wang, Y. Yang, An Error Analysis of Generative Adversarial Networks for Learning Distributions. {\it Journal of Machine Learning Research (JMLR)}, pp. 1-43, (2022).

\bibitem{Zhang} P. Zhang, Q. Liu, D. Zhou, T. Xu, and X. He, On the discrimination - generalization trade-off in GANs. {\it In Proceedings International Conference on Learning Representations (ICLR)}, (2018).

\bibitem{M2} M. Anthony and P. L. Bartlett. {\it Learning in Neural Networks: Theoretical Foundations}, Cambridge University Press, (1999).
	
\bibitem{Pc} P. C. Petersen, Neural Network Theory, University of Vienna. April 18, (2022).
	
\bibitem{Xi} X. Chen, Y. Duan, R. Houthooft, J. Schulman, I. Sutskever, P. Abbeel, infoGAN: Interpretable Representation Learning by Information Maximizing Generative Adversarial Nets. {\it Neural Information Processing Systems (NIPS)}, (2016).
\bibitem{Ji} J. Gui, Z. Sun, Y. Wen, D. Tao, J. Ye, A Review on Generative Adversarial Networks: Algorithms, Theory, and Applications. {\it IEEE Transactions on Knowledge and Data Engineering}, (2023).
\bibitem{Sha} S. Singh, A. Uppal, B. Li, C. Li, M. Zaheer, and B.
Poczos. Nonparametric density estimation under adversarial losses. In Advances in Neural Information Processing Systems, pp. 1024-1057, (2018).

\bibitem{R1} R. M. Dudley. Real analysis and probability. Cambridge University Press, second edition, 2018.

\bibitem{wang} Z. Wang, Q. Guo, S. Sun, C. Xia, The impact of awareness diffusion on SIR-like epidemics in multiplex networks, Applied Mathematics and Computation, Volume 349, pp. 134-147, (2019).

\bibitem{Nian} F. Nian, S. Yao, The epidemic spreading on the multi-relationships network, Applied Mathematics and Computation, pp. 866-873, (2018).
\end{thebibliography}
 \end{document}